\pdfoutput=1
\documentclass{article}
\usepackage[preprint]{neurips_2020}
\usepackage[utf8]{inputenc} 
\usepackage[T1]{fontenc}    
\usepackage{hyperref}       
\usepackage{url}            
\usepackage{booktabs}       
\usepackage{amsfonts}       
\usepackage{nicefrac}       
\usepackage{microtype}      
\usepackage{amsmath}
\usepackage[numbers]{natbib}
\usepackage{graphicx}
\usepackage{subcaption}
\usepackage{float} 
\usepackage{amsthm}
\usepackage{bm}
\usepackage{bbm}
\usepackage{algorithm}
\usepackage{algorithmic}
\usepackage{thm-restate}
\usepackage{color, colortbl}
\usepackage{xcolor}

\definecolor{red}{HTML}{E51400}  
\definecolor{blue}{HTML}{0050EF} 
\definecolor{green}{HTML}{008A00} 
\definecolor{purple}{HTML}{AA00FF} 
\definecolor{dark-red}{rgb}{0.4, 0.15, 0.15}
\definecolor{dark-blue}{rgb}{0.15, 0.15, 0.4}
\definecolor{medium-red}{rgb}{0.5, 0, 0}
\definecolor{medium-blue}{rgb}{0, 0, 0.5}
\definecolor{light-red}{rgb}{0.7, 0, 0}
\definecolor{light-blue}{rgb}{0, 0, 0.7}

\usepackage{ifthen}

\newcommand{\compilehidecomments}{true}
\ifthenelse{ \equal{\compilehidecomments}{true} }{%
	\newcommand{\jinhang}[1]{}
	\newcommand{\xutong}[1]{}
	\newcommand{\carlee}[1]{}
	\newcommand{\wei}[1]{}
}{
	\newcommand{\jinhang}[1]{{\color{green} [Jinhang: #1]}}
	\newcommand{\xutong}[1]{{\color{green} [\text{xutong:} #1]}}
	\newcommand{\carlee}[1]{{\color{red} [Carlee: #1]}}
	\newcommand{\wei}[1]{{\color{blue} [Wei: #1]}}
}

\newcommand{\compilefullversion}{true}
\ifthenelse{\equal{\compilefullversion}{false}}{%
	\newcommand{\OnlyInFull}[1]{}
	\newcommand{\OnlyInShort}[1]{#1}
}{%
	\newcommand{\OnlyInFull}[1]{#1}%
	\newcommand{\OnlyInShort}[1]{}%
}%

\usepackage{hyperref}
\hypersetup{
    colorlinks=true,
    linkcolor=blue
}

\newtheorem{mylem}{Lemma}
\newtheorem{myfact}{Fact}
\newtheorem{mycond}{Condition}

\newcommand{\calF}{\mathcal{F}}
\newcommand{\calN}{\mathcal{N}}

\newcommand{\Ns}{\ensuremath{\calN^{\textnormal{s}}}}

\newcommand{\opt}{{\mathrm{opt}}}

\newcommand{\I}{\mathbb{I}}

\newcommand{\vmu}{\ensuremath{\boldsymbol \mu}}

\newcommand{\E}{\mathop{\mathbb{E{}}}}

\title{Combinatorial Multi-armed Bandits for Resource Allocation}

\author{
  Jinhang Zuo, Carlee Joe-Wong\\
  Carnegie Mellon University\\
  \{jzuo, cjoewong\}@andrew.cmu.edu\\
}
\begin{document}
\maketitle
\begin{abstract}
We study the sequential resource allocation problem where a decision maker repeatedly allocates budgets between resources. Motivating examples include allocating limited computing time or wireless spectrum bands to multiple users (i.e., resources). At each timestep, the decision maker should distribute its available budgets among different resources to maximize the expected reward, or equivalently to minimize the cumulative regret. In doing so, the decision maker should learn the value of the resources allocated for each user from feedback on each user's received reward. For example, users may send messages of different urgency over wireless spectrum bands; the reward generated by allocating spectrum to a user then depends on the message's urgency. We assume each user's reward follows a random process that is initially unknown. We design combinatorial multi-armed bandit algorithms to solve this problem with discrete or continuous budgets. We prove the proposed algorithms achieve logarithmic regrets under semi-bandit feedback.
\end{abstract}

\section{Introduction}\label{sec:intro}

Resource allocation, which generally refers to the problem of distributing a limited budget among multiple entities, is a fundamental challenge that arises in many types of systems, including wireless networks, computer systems, and power grids. Generally, the entity in charge of distributing the budget wishes to do so in an ``optimal'' manner, where ``optimality'' may be defined according to a variety of objectives. In this work, we consider an \emph{online} version of the resource allocation problem, where the objective is not known a priori but can be learned over time based on feedback from the users to whom the budget is allocated.
In the online resource allocation problem, a decision maker (agent) must repeatably distribute its available budgets among different resources (users). Each resource will generate a random reward based a general reward function of the allocated budget and an unknown distribution. As the resource allocation task is repeated over time, the decision maker can gather information about the reward functions and the unknown distributions from observed reward feedback. Its goal is maximize the cumulative total reward, or equivalently, minimize the cumulative \emph{regret} compared to the total achievable award if the reward distributions were known.

In this paper, we introduce offline and online versions of the general resource allocation problem without specifying the exact forms of the reward functions. We assume that the obtained rewards of different resources are independent from each other and that the agent has to balance the tradeoff between exploration and exploitation: as the total budget is limited, the agent needs to intelligently allocate it to not only the resources that may provide high rewards, but also those have not been tried many times yet.
We consider both discrete and continuous budgets, which can be applied to different real-world applications. For example, discrete budgets can be used in the case of allocating computation tasks with the same job size to different servers, while continuous budgets can model the power allocation problem on wireless channels. 

We propose two algorithms, CUCB-DRA and CUCB-CRA, for online discrete and continuous resource allocation, respectively. We adapt the Combinatorial Multi-arm Bandit (CMAB) framework~\cite{chen2016combinatorial} for the online Discrete Resource Allocation (DRA) problem. The proposed CUCB-DRA algorithm considers the action "allocating $a$ budget to resource $k$" as a base arm $(k,a)$. By introducing these base arms, CUCB-DRA does not have to learn the exact form of the reward functions of different resources, and only needs to maintain the upper confidence bounds (UCBs) on the expected rewards of playing the $(k,a)$'s, which are updated by the obtained rewards from each resource in each round. We prove that CUCB-DRA achieves logarithmic regret with the number of rounds $T$. For the online Continuous Resource Allocation (CRA) problem, as the action space becomes infinite, we cannot directly apply the CMAB framework. We propose a CUCB-CRA algorithm integrating CMAB with fixed  discretization~\cite{kleinberg2019bandits} that splits the continuous action space into a discrete one. This discretization technique relies on a Lipschitz condition, which is satisfied by many types of real reward functions. We decompose the cumulative regret into the learning regret and the discretization  error, then choosing the optimized discretization granularity to minimize the sum of them. We show that CUCB-CRA achieves logarithmic regrets with $T$.

We begin the paper by contrasting our approach with prior related work in Section~\ref{sec:related}. We then formulate the offline and online resource allocation problems in Section~\ref{sec:model} and introduce the CUCB-DRA and CUCB-CRA solution algorithms in Sections~\ref{sec:discrete} and~\ref{sec:continues} respectively. We then outline some future research directions and conclude in Section~\ref{sec:conclusion}.
\section{Related Work}\label{sec:related}
The classical resource allocation problem has been extensively studied for decades~\cite{hegazy1999optimization,julian2002qos,joe2013multiresource,devanur2019near}. Recently, the online version of the resource allocation problem has attracted much attention~\cite{devanur2011near,lattimore2015linear,verma2019censored,fontaine2020adaptive}. For example, \cite{lattimore2015linear} introduced the online linear resource allocation problem where the reward functions are assumed to be linear, while \cite{verma2019censored} studied the online resource allocation with censored feedback. \cite{fontaine2020adaptive} considered the online resource allocation problem with concave reward functions. All of these previous works assumed specific types of reward functions, while in this paper, we introduce an online resource allocation framework with \emph{general} reward functions and show that combinatorial bandit techniques can be used to achieve logarithmic solution regret. To best of our knowledge, we are also the first to formally model the discrete and continuous budgets in online resource allocation problems.

Our proposed CUCB-DRA is based the CUCB algorithm in~\cite{chen2016combinatorial}. However, CUCB was designed for the binary action space, which we extend to the finite discrete space, by introducing a new definition of base arms for the online discrete resource allocation problem. CUCB-CRA further extends the action space to an infinite continuous space, by combining the idea of CUCB and fixed discretization in~\cite{kleinberg2019bandits}. Its regret analysis relies on the Lipschitz condition defined in \cite{kleinberg2019bandits} and the 1-norm bounded smoothness condition defined in \cite{wang2017improving}.
\section{Problem Formulation}\label{sec:model}

In this section, we formulate the problem of allocating a fixed budget to multiple resources. We first consider this problem in its offline setting and then consider the online setting, which we focus on for the rest of the paper.

\subsection{Offline Setting}\label{sec:offline}
We consider a general resource allocation problem where a decision maker has access to $K$ different types of resources. The decision maker has to split a total amount of $Q$ divisible budget and allocate $a_k$ budget to each resource $k\in[K]$. For example, the budget may represent compute time, and each resource may represent a user of a particular server. We consider a general reward function $f_k(a_k, X_k)$ of each resource $k$, where $a_k$ is the allocated budget and $X_k$ is a random variable that reflects the random fluctuation of the generated reward.
Notice that the allocated budget $a_k$'s could be either discrete (e.g., $a_k\in\mathbb{N}$) or continuous (e.g., $a_k\in\mathbb{R}_{\geq 0}$), and we denote the feasible action space of $a_k$ as $\mathcal{A}$. 
For the offline setting, we assume the distributions of $X_k$ for all $k\in[K]$ are known in advance and denote them as $D_k$.
Our goal is to maximize the expected total reward collected from all resources. This can be formulated as the optimization problem below:

\begin{equation} \label{eq:newopt}
\begin{aligned}
& \underset{a_{k}}{\text{maximize}}
& & \mathbb{E}\left[\sum_{k=1}^{K} f_k(a_k, X_k)\right] \\
& \text{subject to}
& & \sum_{k=1}^{K} a_{k} \leq Q, a_k \in \mathcal{A}\\
\end{aligned}
\end{equation}
With different reward functions $f_k(a_k, X_k)$ and action spaces $\mathcal{A}$, the hardness of this offline optimization problem varies. For example, it becomes a convex optimization problem if $f_k(a_k, X_k)$ is convex over $a_k$ and $\mathcal{A}$ is a convex set for all $k\in[K]$; at the other extreme, it can also be a NP-hard combinatorial optimization problem when $\mathcal{A}$ is a discrete set. We will not specify the exact form of the optimization problem, but only assume that there exists an offline approximation oracle that can give us an approximate solution with constant approximation ratio. More details of the offline oracle will be discussed in the next section.

\subsection{Online Setting} \label{sec:online}
Now we introduce the online version of the general resource allocation problem, which is a sequential decision making problem. In each round $t$, we allocate $a_{k,t}$ budget to resource $k$ for all $k\in[K]$, subject to the total budget constraint, $\sum_{k=1}^{K} a_{k,t} \leq Q$. We then observe the semi-bandit feedback, which is the reward $f_k(a_{k,t}, X_{k,t})$ from each resource $k$, where $X_{k,t}$ is sampled from an unknown distribution $D_k$.
The total obtained reward is $\sum_{k=1}^{K} f_k(a_{k,t}, X_{k,t})$. Our goal is to accumulate as much total reward as possible through this repeated budget allocation over multiple rounds.

We denote the budget allocation to all resources at round $t$ as $\bm{a}_t = (a_{1,t}, \cdots, a_{K,t})$ and the joint distribution of all independent $X_{k,t}$'s as $\bm{D} = (D_1, \cdots, D_K)$.
The expected total reward obtained in round $t$ can be defined as $r(\bm{a}_t, \bm{D}) = \mathbb{E}\left[\sum_{k=1}^{K} f_k(a_{k,t}, X_{k,t})\right]$. We consider a learning algorithm $\pi$ that makes the budget allocation $\bm{a}^{\pi}_t$ for round $t$. We can then measure the performance of $\pi$ by its (expected) regret, which is the difference in expected cumulative reward between always taking the best offline allocation and taking the budget allocation selected by algorithm $\pi$. Let $\text{opt}(\bm{D}) = \sup_{\bm{a}_t}r(\bm{a}_t, \bm{D})$ denote the expected total reward of the optimal allocation in round $t$. As discussed in the previous section, we assume that there exists an offline $(\alpha,\beta)$-approximation oracle $\mathcal{O}$, which outputs an allocation $\bm{a}_t^{\mathcal{O}}$ such that $\text{Pr}\{r(\bm{a}_t^{\mathcal{O}}, \bm{D}) \geq \alpha \cdot \text{opt}(\bm{D})) \} \geq \beta$, where $\alpha$ is the approximation ratio and $\beta$ is the success probability. Instead of comparing 
with the exact optimal reward, we take the $\alpha \beta$ fraction of it and use the following $(\alpha,\beta)$-approximation regret for $T$ rounds:
\begin{equation}\textstyle\label{eq:regret}
    Reg^{\pi}_{\alpha,\beta}(T;\bm{D}) = T \cdot \alpha \cdot \beta \cdot \text{opt}(\bm{D}) - \sum_{t=1}^T r(\bm{a}_t^{\pi}, \bm{D}),
\end{equation}
In the next two sections, we give solution algorithms for our resource allocation problem that achieve logarithmic $(\alpha,\beta)$-approximation regret when the action space $\mathcal{A}$ is discrete (Section~\ref{sec:discrete}) and continuous (Section~\ref{sec:continues}).
\section{Online Discrete Resource Allocation}\label{sec:discrete}
In this section, we consider the online Discrete Resource Allocation (DRA) problem, where $a_{k,t}$ is chosen from a discrete action space. For example, the $a_{k,t}$ may represent discrete wireless spectrum bands that should be allocated to multiple users. For simplicity, we assume that the action space of $a_{k,t}$ is $\mathcal{A}_{d} = \{0, 1, \cdots, N-1\}$ where $|\mathcal{A}_{d}| = N \leq Q + 1$ and $Q$ is again the available budget. Thus, the full allocation space is $\{\bm{a}_t \mid a_{k,t}\in\mathcal{A}_{d}, \sum_k a_{k,t} \leq Q\}$. In order to solve the online problem introduced in Section~\ref{sec:online}, we adapt a Combinatorial Multi-arm Bandit (CMAB) framework~\cite{chen2016combinatorial}. We maintain a set of base arms $S=\{(k, a) \mid k\in[K], a\in \mathcal{A}_{d}\}$, where the total number of base arms $|S| = KN$. For each base arm $(k,a) \in S$, we denote the expected reward of playing $(k,a)$ as $\mu_{k,a} = \mathbb{E}_{X_{k,t} \sim D_k}\left[f_k(a, X_{k,t})\right]$ and let $\bm{\mu} = (\mu_{k,a})_{(k,a)\in S}$.
We can rewrite the expected total reward obtained in round $t$ as a function of $\bm{a}_t$ and $\bm{\mu}$:
\begin{align}
\nonumber
    r'(\bm{a}_t, \bm{\mu}) &= \mathbb{E}\left[\sum_{k=1}^{K} f_k(a_{k,t}, X_{k,t})\right]\\
    &=\sum_{k=1}^{K}\sum_{a\in\mathcal{A}_{d}} \mu_{k,a}\cdot \mathbbm{1}\{a_{k,t} = a\}, 
\end{align}
which reflects the fact that the expected total reward is the sum of the expected reward of all chosen arms.
\begin{algorithm}[t]
 \caption{CUCB-DRA with offline oracle $\mathcal{O}$}\label{alg:CUCB-DRA}
 \begin{algorithmic}[1]
 \STATE \textbf{Input}: Budget $Q$, Oracle $\mathcal{O}$.
 \STATE For each arm $(k,a) \in S$, $T_{k,a}\leftarrow 0$. \{maintain the total number of times arm $(k,a)$ is played so far.\}
 \STATE For each arm $(k,a) \in S$, $\hat{\mu}_{k,a} \leftarrow 0$. \{maintain the empirical mean of $f_k(a,X_k)$.\}
 \FOR{$t = 1,2,3,\dots$}
    \STATE For each arm $(k,a) \in S, \rho_{k,a} \leftarrow \sqrt{\frac{3\ln t}{2 T_{k,a}}}$. \{the confidence radius, $\rho_{k,a} = +\infty$ if $T_{k,a} = 0$.\}
    \STATE For each arm $(k,a) \in S, \Bar{\mu}_{k,a} = \hat{\mu}_{k,a} + \rho_{k,a}$. \{the upper confidence bound.\}
    \STATE $\bm{a}_{t} \leftarrow \mathcal{O}((\Bar{\mu}_{k,a})_{(k,a)\in S}, Q)$.
    \STATE Take allocation $\bm{a}_{t}$, observe feedback $f_k(a_{i,t},X_{k,t})$'s.
    \STATE For each $k\in [K]$, update $T_{k,a_{k,t}}$ and $\hat{\mu}_{k,a_{k,t}}$: $T_{k,a_{k,t}} = T_{k,a_{k,t}} + 1, \hat{\mu}_{k,a_{k,t}} = \hat{\mu}_{k,a_{k,t}} + (f_k(a_{k,t},X_{k,t})-\hat{\mu}_{k,a_{k,t}}) / T_{k,a_{k,t}}$.
 \ENDFOR
 \end{algorithmic} 
\end{algorithm}

Based on the new parameters $\mu_{k,a}$'s, we propose the CUCB-DRA solution algorithm described in Alg.~\ref{alg:CUCB-DRA}. 
The algorithm maintains the empirical mean $\hat{\mu}_{k,a}$ and a confidence radius $\rho_{k,a}$ for the reward of each arm $(k,a)\in S$.
It feeds the budget $Q$ and all the upper confidence bound $\Bar{\mu}_{k,a}$'s into the offline oracle $\mathcal{O}$ to obtain an allocation $\bm{a}_t$ for round $t$. The confidence radius $\rho_{k,a}$  is large if arm $(k,a)$ is not chosen often ($T_{k,a}$, which denotes the number of times this arm has been chosen, is small). 
We define the reward gap $\Delta_{\bm{a}}{=}\max(0, \alpha \cdot \text{opt}(\bm{D}) - r(\bm{a}, \bm{D}))$ for all feasible allocations $\bm{a} \in \mathcal{A}_{d}^K, \sum_{k=1}^K a_k \leq Q$. For each arm $(k,a)$, we define $\Delta^{k,a}_{\min} = \inf_{\bm{a} \in \mathcal{A}_{d}^K:a_k = a, \Delta_{\bm{a}} > 0} \Delta_{\bm{a}}$, $\Delta^{k,a}_{\max} = \sup_{\bm{a} \in \mathcal{A}_{d}^K:a_k = a, \Delta_{\bm{a}} > 0} \Delta_{\bm{a}}$. 
We define $\Delta_{\min} = \min_{(k,a)\in S}\Delta^{k,a}_{\min}$ and $\Delta_{\max} = \max_{(k,a) \in S}\Delta^{k,a}_{\max}$. We then provide the regret bounds of the CUCB-DRA algorithm. 
\begin{restatable}{theorem}{thmOverDRA}\label{thm:DRA}
For the CUCB-DRA algorithm on an online DRA problem with a bounded smoothness constant $B\in\mathbb{R}^+$~\cite{wang2017improving}

\begin{enumerate}
  \item if $\Delta_{\min} > 0$, we have a distribution-dependent bound
\begin{align}\textstyle
    \text{Reg}_{\alpha,\beta}(T, \bm{D}) \leq  \sum_{(k,a) \in S} \frac{48B^2Q\ln T}{\Delta^{k,a}_{\min}}  + 2BKN + \frac{\pi^2}{3}\cdot KN\cdot\Delta_{\max}
\end{align}
  \item we have a distribution-independent bound
  \begin{align}\textstyle
    \text{Reg}_{\alpha,\beta}(T, \bm{D}) \leq  14B\sqrt{QKNT\ln T}+ 2BKN + \frac{\pi^2}{3}\cdot KN\cdot\Delta_{\max}.
\end{align}
\end{enumerate}
\end{restatable}
Notice that our regret results hold for any finite discrete action space $\mathcal{A}_{d}$. The only change in the regrets will be the replacement of $N$ with the actual $|\mathcal{A}_{d}|$.
The full proof of the above theorem is provided in the technical report. We rely on the following properties of $r'(\bm{a}_t, \bm{\mu})$, which are required by the general CMAB framework in~\cite{wang2017improving}, to bound the regret.

\begin{mycond}(Monotonicity). \label{cond: over mono}
	The reward $r'(\bm{a}_t, \bm{\mu})$ satisfies monotonicity, if
	for any allocation $\bm{a}_t$, any two vectors $\bm{\mu} = (\mu_{k,a})_{(k,a)\in S}$, $\bm{\mu}' = (\mu'_{k,a})_{(k,a)\in S}$, we have $r'(\bm{a}_{t}, \bm{\mu}) \leq r'(\bm{a}_{t}, \bm{\mu}')$, if $\mu_{k,a} \leq \mu'_{k,a}$ for all $(k,a)\in S$.
\end{mycond}

\begin{mycond}(1-Norm Bounded Smoothness).\label{cond: over 1-norm}
	The reward $r'(\bm{a}_t, \bm{\mu})$ satisfies the 1-norm bounded smoothness condition, if there exists $B\in \mathbb{R}^+$ (referred as the bounded smoothness constant) such that, 
	for any allocation $\bm{a}_t$, and any two vectors $\bm{\mu} = (\mu_{k,a})_{(k,a)\in S}$, $\bm{\mu}' = (\mu'_{k,a})_{(k,a)\in S}$, we have $|r'(\bm{a}_t, \bm{\mu}) - r'(\bm{a}_t, \bm{\mu}')| \leq B \sum_{(k,a)\in S} |\mu_{k,a} - \mu_{k,a}'|$.
\end{mycond}
It is easy to see that both properties hold from the definition of $r'(\bm{a}_t, \bm{\mu})$ in Eq.~\eqref{eq:regret_discretization}.
\section{Online Continuous Resource Allocation}\label{sec:continues}
In Section~\ref{sec:discrete}, we propose an algorithm to solve the online resource allocation problem with discrete action space ~$\mathcal{A}$. However, in many real-world resource allocation applications, the budget can be continuous, i.e., $\mathcal{A}$ is an infinite continuous space. In this section, we study the online Continuous Resource Allocation (CRA) problem, where $a_{k,t}$ is chosen from a continuous space $\mathcal{A}_c = [0, Q]$. For example, the actions $a_{k,t}$ may be amounts of electricity that a smart power grid pulls from different electric vehicle charging stations. The  full allocation space then becomes $\{\bm{a}_t \mid a_{k,t}\in\mathcal{A}_{c}, \sum_k a_{k,t} \leq Q\}$. As in Section~\ref{sec:discrete}'s discrete setting, we still define the set of base arms as $S=\{(k, a) \mid k\in[K], a\in \mathcal{A}_{c}\}$, but different from the discrete setting, now we have to maintain infinite number of base arms. Thus, we cannot directly apply the CMAB framework. However, we can use a simple but powerful technique called fixed discretization~\cite{kleinberg2019bandits}, with the assumption that the reward function $f_k(a_{k,t},X_{k,t})$ satisfies a Lipschitz condition:
\begin{equation}
    |f_k(a,X_{k,t}) - f_k(b,X_{k,t})| \leq L\cdot |a-b|,
\end{equation}
where $L$ is the Lipschitz constant known to the algorithm. This condition is satisfied by many realistic reward functions, e.g., $f_k(a, X_{k,t}) = \max\left\{a - X_{kt}, 0\right\}$, which represents a reward of 0 if the allocated budget $a$ does not meet a requirement $X_{kt}$, with linearly increasing reward otherwise.

We consider a discretization of $\mathcal{A}_c$ and denote it as $\widetilde{\mathcal{A}_c}$. Define $\text{opt}_{\mathcal{A}}(\bm{D}) = \sup_{a_{k,t} \in \mathcal{A}}r(\bm{a}_t, \bm{D})$ as the maximum expected total reward under action space $\mathcal{A}$ and distribution $\bm{D}$. We can decompose the cumulative regret in Eq.~\eqref{eq:regret} as:
\begin{align}\label{eq:regret_discretization}
    Reg^{\pi}_{\alpha,\beta}(T;\bm{D}) &= \underbrace{T \cdot \alpha \cdot \beta \cdot \text{opt}_{\widetilde{\mathcal{A}_c}}(\bm{D}) - \sum_{t=1}^T r(\bm{a}_t^{\pi}, \bm{D})}_{\text{\textcircled{1}}}
    + \underbrace{T \cdot \alpha \cdot \beta \cdot \left(\text{opt}_{\mathcal{A}_c}(\bm{D}) - \text{opt}_{\widetilde{\mathcal{A}_c}}(\bm{D})\right)}_{\text{\textcircled{2}}},
\end{align}
where \textcircled{\small 1} is the learning regret under action space $\widetilde{\mathcal{A}_c}$ and \textcircled{\small 2} is the discretization error. Both of them depend on the discretization space $\widetilde{\mathcal{A}_c}$. 

\begin{algorithm}[t]
 \caption{CUCB-CRA with offline oracle $\mathcal{O}$}\label{alg:CUCB-CRA}
 \begin{algorithmic}[1]
 \STATE \textbf{Input}: Budget $Q$, Lipschitz constant $L$, Time horizon $T$, Oracle $\mathcal{O}$.
 \STATE Let $\epsilon = (\frac{B^2 Q^2 \ln T}{L^2 K T})^{\frac{1}{3}}$, discretize $\mathcal{A}_c = [0, Q]$ into $\widetilde{\mathcal{A}_c} = \{0, \epsilon, 2\epsilon, \cdots, (N-1)\cdot\epsilon\}$.
 \STATE Run CUCB-DRA algorithm with discrete action space $\widetilde{\mathcal{A}_c}$.
 \end{algorithmic} 
\end{algorithm}

We propose a CUCB-CRA algorithm that makes a uniform discretization of the budget for each resource. It divides the original action space $\mathcal{A}_c=[0, Q]$ into intervals of fixed length $\epsilon = \frac{Q}{N-1}$, so that $\widetilde{\mathcal{A}_c}$ consists of $N$ multiples of $\epsilon$, i.e., $\widetilde{\mathcal{A}_c} = \{0, \epsilon, 2\epsilon, \cdots, (N-1)\cdot\epsilon\}$.
With the Lipschitz condition, it is easy to see that $\text{\textcircled{\small 2}} \leq T \cdot \alpha \cdot \beta \cdot LK\epsilon$. For \textcircled{\small 1}, it can be viewed as the regret of a discrete resource allocation problem discussed in Section~\ref{sec:discrete}. where the number of base arms is still $KN$. Based on Theorem \ref{thm:DRA}, we know the regret in \textcircled{\small 1} is in the order of $O(B\sqrt{QKNT\ln T})$. Choosing $\epsilon = (\frac{B^2 Q^2 \ln T}{L^2 K T})^{\frac{1}{3}}$, the regret in Eq.\eqref{eq:regret_discretization} is minimized and we have
\begin{restatable}{theorem}{thmOverCRA}\label{thm:CRA}
For the CUCB-CRA algorithm on an online CRA problem with a bounded smoothness constant $B$, we have a distribution-independent bound
  \begin{align}
\nonumber
    \text{Reg}_{\alpha,\beta}(T, \bm{D}) \leq O((BQK)^{\frac{2}{3}} L^{\frac{1}{3}} T^{\frac{2}{3}} (\ln T)^{\frac{1}{3}}).
\end{align}
\end{restatable}

We note that Theorem~\ref{thm:CRA}'s distribution-independent regret bound is looser than that of Theorem~\ref{thm:DRA} for the discrete resource allocation problem, by a factor of $O\left(T/\ln T\right)^{\frac{1}{6}}$. This factor primarily stems from the additional regret due to the fixed discretization in the continuous case. Adaptive discretization methods, e.g., as proposed in~\cite{kleinberg2019bandits,nika2020contextual} for other MAB problems with continuous arms, may allow further reduction of the regret.

\section{Conclusion and Future Work}\label{sec:conclusion}

In this work, we consider a general resource allocation problem where a fixed budget must be divided among multiple resources. We introduce offline and online versions of the problem and give two solution algorithms to the online problem when discrete and continuous actions are possible. Unlike many previous works that consider online resource allocation, our formulation is quite general and can, at a high level, encompass formulations as diverse as a computing server dividing time among its users or a smart grid procuring energy from multiple electric vehicle charging stations. Leveraging insights from combinatorial multi-armed bandit frameworks, we show that both of our solution algorithms achieve logarithmic regret with $T$, the number of rounds.

Much future work on this resource allocation problem remains, including examining solution algorithms under different levels of reward feedback and the optimality of the CUCB-DRA and CUCB-CRA algorithms that we provide. We have assumed in this paper that the system's feedback includes only the reward received by each user. For some types of reward functions $f_k$, however, the received reward may additionally offer information as to the value of $X_k$. Algorithms that exploit such information will likely give lower regret bounds than those achieved by CUCB-DRA and CUCB-CRA. We can gain further insights into CUCB-DRA's and CUCB-CRA's performance by comparing them to lower bounds on the achievable regret under different feedback levels.

\bibliographystyle{unsrtnat}  
\bibliography{references} 
\clearpage
\appendix
\section*{Appendix}\label{sec:supplementary}

\section{Proof of Theorem \ref{thm:DRA}}
\begin{proof}
\begin{myfact}[Hoeffding's Inequality]\label{fact:hoeffding}
	Let $X_1, \cdots , X_n$ be independent and identically distributed random variables with
	common support $[0,1]$ and mean $\mu$. 
	Let
	$Y=X_1+\cdots+X_n$. Then for all $\delta \geq0$,
	\[\Pr\{|Y-n\mu|\geq \delta \} \leq 2e^{-2\delta^2/n}.
	\]
\end{myfact}

\begin{mylem}
	\label{lem:ns}
	Let $\Ns_t$ be the event that at the beginning of round $t$, for every arm $(k,a) \in S$, $|\hat\mu_{k,a,t-1} - \mu_{k,a}|<\rho_{k,a, t}$. Then for each round $t \ge 1$, $\Pr\{\lnot \Ns_t\}\le 2|S|t^{-2}$.
\end{mylem}
\begin{proof}
	For each round $t \ge 1$, we have
	\begin{align}
	\Pr\{\lnot \Ns_t\} & = \Pr\left\{\exists (k,a)\in S, |\hat\mu_{k,a,t-1} - \mu_{k,a}| \ge \sqrt{\frac{3\ln t}{2T_{k,a, t-1}}}\right\} \nonumber \\
	& \le\sum_{(k,a)\in S} \Pr\left\{|\hat\mu_{k,a,t-1} - \mu_{k,a}| \ge \sqrt{\frac{3\ln t}{2T_{k,a, t-1}}}\right\}. \nonumber \\
	& =\sum_{(k,a)\in S} \sum_{s=1}^{t-1}\Pr\left\{T_{k,a,t-1}=s, |\hat\mu_{k,a,t-1} - \mu_{k,a}| \ge \sqrt{\frac{3\ln t}{2T_{k,a, t-1}}}\right\}.  \label{eq:nicesample1}
	\end{align}
	When $T_{k,a,t-1}=s$, $\hat\mu_{k,a,t-1}$ is the average of $s$ i.i.d. random outcomes of arm $(k,a)$. With the Hoeffding's Inequality (Fact~\ref{fact:hoeffding}), we have
	\begin{align}
	\Pr\left\{T_{k,a,t-1}=s, |\hat\mu_{k,a,t-1} - \mu_{k,a}| \ge \sqrt{\frac{3\ln t}{2T_{k,a, t-1}}}\right\} \le 2 t^{-3},
	\label{eq:nicesample2}
	\end{align}
	Combining Eq.\eqref{eq:nicesample1} and \eqref{eq:nicesample2}, we have $\Pr\{\lnot \Ns_t\}\le 2|S|t^{-2}$.
\end{proof}
	We generally follow the proof of Theorem 4 in \cite{wang2017improving}, with the different definition of the base arm.
    We first introduce a positive real number $M_{k,a}$ for each arm $(k,a)$.
	Let $\calF_t$ be the event $\{r'(\bm{a}_t, \bar{\bm{\mu}}) < \alpha\cdot \opt(\bar\vmu)\}$, which represents the oracle fails in round $t$.
	Define $M_{\bm{a}} = \max_{(k,a)\in \bm{a}} M_{k,a}$ for each action $\bm{a}$.
	Define
	$$\kappa_{T}(M, s) = \begin{cases}
	2B, &\mbox{if } s=0,\\
	2B\sqrt{\frac{6 \ln T}{s}}, &\mbox{if } 1\le s\le \ell_{T}(M),\\
	0, &\mbox{if } s \ge \ell_{T}(M)+1,
	\end{cases}$$
	where
	$$\ell_{T}(M)=\left\lfloor\frac{24 B^2 Q^2 \ln T}{M^2}\right\rfloor.$$
	We then show that if $\{\Delta_{\bm{a}_t} \ge M_{\bm{a}_t}\}$, $\lnot \calF_t$ and $\Ns_t$ hold, we have
	\begin{equation}
	\Delta_{\bm{a}_t} \le \sum_{(k,a)\in \bm{a}_t}\kappa_T(M_{k,a}, T_{k,a, t-1}).
	\label{eq:1-norm.nontriggering.kappa}
	\end{equation}
	The right hand side of the inequality is non-negative, so it holds naturally if $\Delta_{\bm{a}_t}=0$. We only need to consider $\Delta_{\bm{a}_t}>0$.
	By $\Ns_t$ and $\lnot \calF_t$, we have
	$$r'(\bm{a}_t, \bar\vmu_t)\ge \alpha\cdot \opt(\bar\vmu_t)\ge \alpha\cdot \opt(\vmu) = r'(\bm{a}_t,\vmu) + \Delta_{\bm{a}_t}$$
	Then by Condition~\ref{cond: over 1-norm},
	$$\Delta_{\bm{a}_t}\le r'(\bm{a}_t, \bar\vmu_t)-r'(\bm{a}_t, \vmu)\le B\sum_{(k,a)\in \bm{a}_t} (\bar\mu_{k,a, t} - \mu_{k,a}).$$
	We are going to bound $\Delta_{\bm{a}_t}$ by bounding $\bar\mu_{k,a, t} - \mu_{k,a}$. We have
	\begin{align}
	\Delta_{\bm{a}_t}
	&\le B\sum_{(k,a)\in \bm{a}_t} (\bar\mu_{k,a, t} - \mu_{k,a})\nonumber\\
	&\le -M_{\bm{a}_t} + 2B\sum_{(k,a)\in \bm{a}_t} (\bar\mu_{k,a, t} - \mu_{k,a})\nonumber\\
	&\le 2B\sum_{(k,a)\in \bm{a}_t} \left[(\bar\mu_{k,a, t} - \mu_{k,a}) - \frac{M_{\bm{a}_t}}{2BQ}\right]\nonumber\\
	&\le 2B\sum_{(k,a)\in \bm{a}_t} \left[(\bar\mu_{k,a, t} - \mu_{k,a}) - \frac{M_{k,a}}{2BQ}\right]. \label{eq:nontriggering.transform}
	\end{align}
	By $\Ns_t$, we have $\bar\mu_{k,a, t} - \mu_{k,a} \le 2\rho_{k,a, t}$, so
	$$(\bar\mu_{k,a, t} - \mu_{k,a}) - \frac{M_{k,a}}{2BQ} \le 2\rho_{k,a, t} - \frac{M_{k,a}}{2BQ} \le 2\sqrt{\frac{3\ln T}{2T_{k,a, t-1}}} - \frac{M_{k,a}}{2BQ}.$$
	If $T_{k,a, t-1}\le \ell_T(M_{k,a})$, we have $(\bar\mu_{k,a, t} - \mu_{k,a}) - \frac{M_{k,a}}{2BQ} \le 2\sqrt{\frac{3\ln T}{2T_{k,a, t-1}}} \le \frac{1}{2B}\kappa_T(M_{k,a}, T_{k,a, t-1})$.
	If $T_{k,a, t-1}\ge \ell_T(M_{k,a})+1$, then $2\sqrt{\frac{3\ln T}{2T_{k,a,t-1}}} \le \frac{M_{k,a}}{2BQ}$, so $(\bar\mu_{k,a, t} - \mu_{k,a}) - \frac{M_{k,a}}{2BQ} \le 0 = \frac{1}{2B}\kappa_T(M_{k,a}, T_{k,a, t-1})$.
	In conclusion, we have
	$$\eqref{eq:nontriggering.transform} \le \sum_{(k,a)\in \bm{a}_t} \kappa_T(M_{k,a}, T_{k,a, t-1}).$$
	
	Then for all rounds,
	\begin{align*}
	\sum_{t=1}^T \I(\{\Delta_{\bm{a}_t} \ge M_{\bm{a}_t}\} \land \lnot \calF_t \land \Ns_t) \cdot \Delta_{\bm{a}_t}
	&\le \sum_{t=1}^T \sum_{(k,a)\in \bm{a}_t} \kappa_T(M_{k,a}, T_{k,a, t-1})\\
	&= \sum_{(k,a)\in S} \sum_{s=0}^{T_{k,a, T}} \kappa_T(M_{k,a}, s)\\
	&\le \sum_{(k,a)\in S} \sum_{s=0}^{\ell_{T}(M_{k,a})} \kappa_T(M_{k,a}, s)\\
	&= 2B|S| + \sum_{(k,a)\in S} \sum_{s=1}^{\ell_{T}(M_{k,a})} 2B\sqrt{\frac{6\ln T}{s}}\\
	&\le 2B|S| + \sum_{(k,a)\in S} \int_{s=0}^{\ell_{T}(M_{k,a})} 2B\sqrt{\frac{6\ln T}{s}} \mathrm{d} s\\
	&\le 2B|S| + \sum_{(k,a)\in S} 4B\sqrt{6\ln T \ell_T(M_{k,a})}\\
	&\le 2B|S| + \sum_{(k,a)\in S} \frac{48B^2Q\ln T}{M_{k,a}}.
	\end{align*}
	So
	\begin{align*}
	Reg(\{\Delta_{\bm{a}_t} \ge M_{\bm{a}_t}\} \land \lnot \calF_t \land \Ns_t)
	& = \E\left[\sum_{t=1}^T \I(\{\Delta_{\bm{a}_t} \ge M_{\bm{a}_t}\} \land \lnot \calF_t \land \Ns_t) \cdot \Delta_{\bm{a}_t}\right] \\
	& \le 2B|S| + \sum_{(k,a)\in S} \frac{48B^2Q\ln T}{M_{k,a}}.
	\end{align*}
	
	By Lemma~\ref{lem:ns}, $\Pr\{\lnot \Ns_t\} \le 2|S|t^{-2}$.
	Then we have
	$$Reg(\lnot \Ns_t) \le \sum_{t=1}^T 2|S|t^{-2} \cdot \Delta_{\max} \le \frac{\pi^2}{3} |S|\cdot \Delta_{\max},$$
	$$Reg(\calF_t) \le (1-\beta)T \cdot \Delta_{\max}.$$
	With these two bounds, we have
	\begin{align*}
	Reg(\{\})
	&\le Reg(\calF_t) + Reg(\lnot \Ns_t)
	+ Reg(\{\Delta_{\bm{a}_t} \ge M_{\bm{a}_t}\} \land \lnot \calF_t \land \Ns_t) + Reg(\Delta_{\bm{a}_t} < M_{\bm{a}_t})\\
	&\le (1-\beta)T \cdot \Delta_{\max} + \frac{\pi^2}{3} |S|\cdot \Delta_{\max}
	+ 2B|S| + \sum_{(k,a)\in S} \frac{48B^2Q\ln T}{M_{k,a}} + Reg(\Delta_{\bm{a}_t} < M_{\bm{a}_t}).
	\end{align*}
	Since $Reg_{\alpha, \beta}(T, \bm{D}) = Reg(\{\}) - (1-\beta)T\cdot \Delta_{\max}$,
	$$Reg_{\alpha, \beta}(T, \bm{D}) \le \frac{\pi^2}{3}|S| \cdot \Delta_{\max}
	+ 2B|S| + \sum_{(k,a)\in S} \frac{48B^2Q\ln T}{M_{k,a}} + Reg(\Delta_{\bm{a}_t} < M_{\bm{a}_t}).$$
	For the distribution-dependent bound, take $M_{k,a}=\Delta_{\min}^{k,a}$,
	then $Reg(\Delta_{\bm{a}_t} < M_{\bm{a}_t}) = 0$ and we have
	$$Reg_{\alpha, \beta}(T, \bm{D}) \le \sum_{(k,a)\in S} \frac{48B^2Q\ln T}{\Delta_{\min}^{k,a}} + 2BKN + \frac{\pi^2}{3}\cdot KN \cdot \Delta_{\max}.$$
	For the distribution-independent bound, take $M_{k,a}=M=\sqrt{(48B^2QKN\ln T)/T}$,
	then $Reg(\Delta_{\bm{a}_t} < M_{\bm{a}_t}) \le TM$ and we have
	\begin{align*}
	Reg_{\alpha, \beta}(T, \bm{D})
	&\le \sum_{(k,a)\in S} \frac{48B^2Q\ln T}{M_{k,a}} + 2BKN + \frac{\pi^2}{3}\cdot KN \cdot \Delta_{\max} + Reg(\Delta_{\bm{a}_t} < M_{\bm{a}_t})\\
	&\le \frac{48B^2QKN\ln T}{M} + 2BKN + \frac{\pi^2}{3}\cdot KN \cdot \Delta_{\max} + TM\\
	&= 2\sqrt{48B^2QKNT\ln T} + \frac{\pi^2}{3} \cdot KN \cdot \Delta_{\max} + 2BKN\\
	&\le 14B\sqrt{QKNT\ln T} + \frac{\pi^2}{3} \cdot KN \cdot \Delta_{\max} + 2BKN.
	\end{align*}
\end{proof}
\end{document}